\documentclass{article}

\usepackage{microtype}
\usepackage{graphicx}
\usepackage{subcaption}
\usepackage{booktabs} 

\usepackage{hyperref}

\usepackage[accepted]{icml2026}

\usepackage{amsmath}
\usepackage{amssymb}
\usepackage{mathtools}
\usepackage{amsthm}
\usepackage{multirow}

\usepackage[capitalize,noabbrev]{cleveref}

\theoremstyle{plain}
\newtheorem{theorem}{Theorem}
\newtheorem{proposition}{Proposition}
\newtheorem{lemma}{Lemma}

\theoremstyle{definition}
\newtheorem{definition}{Definition}
\newtheorem{assumption}{Assumption}
\theoremstyle{remark}
\newtheorem{remark}{Remark}

\newcommand{\C}[1]{{\mathcal{#1}}} %
\newcommand{\B}[1]{{\mathbb{#1}}} %
\newcommand{\BF}[1]{{\mathbf{#1}}} %

\newcommand{\E}{\mathbb{E}}
\newcommand{\R}{\mathbb{R}}
\newcommand{\K}{\mathcal{K}}

\newcommand{\bx}{\BF{x}}
\newcommand{\hbx}{\hat{\bx}}
\newcommand{\by}{\BF{y}}

\newcommand{\bu}{\BF{u}}
\newcommand{\bv}{\mathbf{v}}

\newcommand{\CO}{\C{O}}
\newcommand{\BO}{\BF{O}}

\newcommand{\oo}{\BF{o}}

\newcommand{\bzero}{\mathbf{0}}
\newcommand{\bone}{\mathbf{1}}
\newcommand{\Fcal}{\mathcal{F}}
\newcommand{\grad}{\nabla}

\newcommand{\gfrak}{\mathfrak{g}}

\usepackage[textsize=tiny]{todonotes}

\icmltitlerunning{Upper-Linearizability of Online Non-Monotone DR-Submodular Maximization over Down-Closed Convex Sets}
\usepackage{enumitem}
\begin{document}

\twocolumn[
  \icmltitle{Upper-Linearizability of Online Non-Monotone DR-Submodular Maximization over Down-Closed Convex Sets}

  \begin{icmlauthorlist}
    \icmlauthor{Yiyang Lu}{purdue}
    \icmlauthor{Hareshkumar Jadav}{iiti}
    \icmlauthor{Mohammad Pedramfar}{mila}
    \icmlauthor{Ranveer Singh}{iiti}
    \icmlauthor{Vaneet Aggarwal}{purdue}
  \end{icmlauthorlist}

  \icmlaffiliation{purdue}{Purdue University, West Lafayette, IN, USA}
  \icmlaffiliation{iiti}{IIT Indore, MP, India}
  \icmlaffiliation{mila}{Mila - Quebec AI Institute/McGill University, Montreal, QC, Canada}

  \icmlcorrespondingauthor{Yiyang Lu}{yiyanglu@purdue.edu}
  \icmlcorrespondingauthor{Vaneet Aggarwal}{vaneet@purdue.edu}
  
  \icmlkeywords{Machine Learning, ICML}

  \vskip 0.3in
]



\printAffiliationsAndNotice{}  

\begin{abstract}

We study online maximization of non-monotone Diminishing-Return(DR)-submodular functions over down-closed convex sets, a regime where existing projection-free online methods suffer from suboptimal regret and limited feedback guarantees.  
Our main contribution is a new structural result showing that this class is $1/e$-linearizable under carefully designed exponential reparametrization, scaling parameter, and surrogate potential, enabling a reduction to online linear optimization.  
This allows us to obtain first non-Frank-Wolfe type algorithms for this setting that obtain an approximation coefficient better than $1/4$.
Moreover, the linearization framework allows us to move beyond offline optimization.
As a result, we obtain $O(T^{1/2})$ static regret with a single gradient query per round and unlock adaptive and dynamic regret guarantees, together with improved rates under semi-bandit, bandit, and zeroth-order feedback.  
Across all feedback models, our bounds strictly improve the state of the art.

\end{abstract}

\section{Introduction}

Online optimization of submodular and DR-submodular functions has become a central primitive in machine learning, with applications in mean-field inference, revenue maximization, influence maximization, supply chain management, power network reconfiguration, and experimental design \citep{bian2019optimal, ito2016large, gu2023profit, aldrighetti2021costs, mishra2017comprehensive, li2023experimental}.
In these problems, an algorithm repeatedly selects actions from a convex domain while an adversary reveals a reward function, and performance is measured through notions of static, adaptive, or dynamic regret.
A long line of work has developed projection-free algorithms, typically based on Frank--Wolfe or boosting-style updates \citep{fazel2023fast,chen2018online,zhang2022stochastic,pedramfar2023unified,zhang2024boosting}.

While most of the study in DR-submodular optimization is for monotone objectives \citep{hassani2017gradient,fazel2023fast,chen2018online,zhang2022stochastic}, the non-monotone objective plays an important role in many applications such as price optimization, social networks recommendation, and budget allocation \citep{ito2016large,gu2023profit,alon2012optimizing}. In this work, we study non-monotone DR-submodular maximization over down-closed convex sets, which remains particularly challenging.
Down-closed domains include box constraints, knapsack polytopes, and intersections of matroids, and arise naturally in resource allocation and coverage problems. In this regime, even the best approximation ratio for online optimization remains open \citep{buchbinder2024constrained}, while the best achievable approximation rate is $1/e$ \citep{thang2021online,zhang2023online,pedramfar2023unified}. Further, for this regime, existing projection-free methods either require multiple oracle queries per round or only achieve suboptimal regret rates such as $O(T^{2/3})$ \cite{pedramfar2024unified}, and essentially no results were known for adaptive or dynamic regret.

Recent work in \citep{pedramfar2024linear} introduced the notion of \emph{linearizable} function classes and showed how such a structure enables a generic reduction from online linear optimization to a wide family of non-convex problems, including several DR-submodular settings.
While powerful, these results does not cover the non-monotone down-closed case with efficient query complexity and $\CO(T^{1/2})$regret. In this paper, we close this gap by establishing a new structural characterization for online non-monotone DR-submodular maximization over down-closed convex sets.
Our main technical contribution is to show that this class is \emph{$1/e$-linearizable} under a carefully designed exponential reparametrization and surrogate potential.

\textbf{Technical Novelty.} 
Achieving this result requires overcoming two fundamental theoretical barriers. First, establishing the structural $1/e$-linearizability reduction (Theorem~\ref{thm:main}) requires carefully managing the non-monotone penalty over down-closed sets. We achieve this by designing a novel surrogate potential $F(\cdot)$ that acts as an integrating factor, heavily weighting the early stages of the trajectory to enable an exact mathematical cancellation during the integration-by-parts analysis. Second, to translate this structure into an efficient online algorithm, we bypass the heavy computational cost of standard continuous greedy approximations by introducing a Jacobian-corrected gradient estimator (BQND, Algorithm ~\ref{alg:bqnd}). Unlike naive sampling, this estimator explicitly incorporates the curvature of our exponential mapping ($h(x) = 1 - e^{-x}$) into the gradient query, constructing unbiased linear surrogates of the non-monotone objective using a strict single-query budget. Together, these innovations completely decouple the non-convex oracle queries from the constraint handling, allowing our framework to accept any efficient regret-minimizing algorithm for linear functions as a base learner.

Once this structure is in place, a broad collection of algorithmic guarantees follow as immediate consequences through existing regret-transfer principles in \citep{pedramfar2024linear}.
We obtain the first projection-free online algorithms achieving $O(T^{1/2})$ static regret with only a single gradient query per round.
Moreover, the same framework unlocks adaptive and dynamic regret guarantees in adversarial environments, as well as improved rates under semi-bandit, bandit, and zeroth-order full-information feedback.
Across all feedback models, our results strictly improve the previously best-known bounds, as summarized in Table~\ref{tab:comparison_master}.

\begin{table*}[t]
\centering
\caption{Comparison of Regret Bounds for Online Non-Monotone DR-Submodular Maximization over Down-Closed Sets. \textbf{Oracle} denotes the feedback type ($\nabla F$: Gradient, $F$: Value). Our framework is the \textbf{first} to provide Adaptive and Dynamic regret guarantees while achieving $\CO(T^{1/2})$ static regret in the $O(1)$ query regime. \citet{thang2021online} achieves $\CO(T^{3/4})$ regret with $T^{3/4}$ queries per round when $\beta=\frac{3}{4}$. \citet{zhang2023online} achieves $\CO(T^{1/2})$ regret with $T^{3/2}$ queries per round when $\beta=\frac{3}{2}$, but only achieves $\CO(T^{4/5})$ regret with $\CO(1)$ queries.}
\label{tab:comparison_master}
\vskip 0.15in
\begin{small}
\begin{sc}
\resizebox{\textwidth}{!}{%
\begin{tabular}{ccccccccc}
\toprule
\multirow{2}{*}{\textbf{Oracle}} & \multirow{2}{*}{\textbf{Feedback}} & \multirow{2}{*}{\textbf{Reference}} & \multirow{2}{*}{\textbf{Approx.}} & \multirow{2}{*}{\textbf{Queries}} & \multicolumn{3}{c}{\textbf{Regret Guarantees}} \\
\cmidrule{6-8}
& & & & & \textbf{Static} & \textbf{Adaptive} & \textbf{Dynamic ($\times \sqrt{1+P_T}$)} \\
\midrule
\multirow{9}{*}{$\nabla F$} 
 & \multirow{5}{*}{Full Info} 
   & \citep{thang2021online} & $1/e$ & $T^{\beta},\beta\in[0,\frac{3}{4}]$ & $O(T^{1-\beta/3})$ & -- & -- \\
 & &  \cite{zhang2023online} & $1/e$ & $T^{\beta},\beta\in[0,\frac{3}{2}]$ & $O(T^{1-\beta/3})$ & -- & -- \\
 & &\cite{zhang2023online} & $1/e$ & $1$ & $O(T^{4/5})$ & -- & -- \\
 & & \citep{pedramfar2024unified} & $1/e$ & $1$ & $O(T^{2/3})$ & -- & -- \\
 & & \textbf{This Paper} & $\mathbf{1/e}$ & \textbf{1} & $\mathbf{O(T^{1/2})}$(Prop.~\ref{prop:static-regret}) & $\mathbf{O(T^{1/2})}$(Prop.~\ref{prop:adaptive-regret}) & $\mathbf{\tilde{O}(T^{1/2})}$(Prop.~\ref{prop:dynamic-regret}) \\
 \cmidrule{2-8}
 & \multirow{4}{*}{Semi-Bandit} 
   & \cite{zhang2023online} & $1/e$ & $1$ & $O(T^{4/5})$ & -- & -- \\
 & & \citep{pedramfar2024unified} & $1/e$ & $1$ & $O(T^{3/4})$ & -- & -- \\
 & & \textbf{This Paper} & $\mathbf{1/e}$ & \textbf{1} & $\mathbf{O(T^{2/3})}$ (Prop.~\ref{prop:semi_bandit}) & $\mathbf{O(T^{2/3})}$(Prop.~\ref{prop:semi_bandit}) & $\mathbf{\tilde{O}(T^{2/3})}$(Prop.~\ref{prop:dynamic_limited}) \\
\midrule
\multirow{6}{*}{$F$} 
 & \multirow{3}{*}{Full Info} 
   & \cite{pedramfar2024unified} & $1/e$ & $1$ & $O(T^{4/5})$ & -- & -- \\
 & & \textbf{This Paper} & $\mathbf{1/e}$ & \textbf{1} & $\mathbf{O(T^{3/4})}$(Prop.~\ref{prop:zo_full_info}) & $\mathbf{O(T^{3/4})}$(Prop.~\ref{prop:zo_full_info}) & $\mathbf{\tilde{O}(T^{3/4})}$(Prop.~\ref{prop:dynamic_limited}) \\
 \cmidrule{2-8}
 & \multirow{4}{*}{Bandit} 
   & \cite{zhang2023online} (Det.)$^*$ & $1/e$ & $1$ & $O(T^{8/9})$ & -- & -- \\
 & & \cite{pedramfar2024unified} & $1/e$ & $1$ & $O(T^{5/6})$ & -- & -- \\
 & & \textbf{This Paper} & $\mathbf{1/e}$ & \textbf{1} & $\mathbf{O(T^{4/5})}$ (Prop.~\ref{prop:bandit}) & $\mathbf{O(T^{4/5})}$ (Prop.~\ref{prop:bandit}) & $\mathbf{\tilde{O}(T^{4/5})}$(Prop.~\ref{prop:dynamic_limited}) \\
\bottomrule
\end{tabular}%
}
\end{sc}
\end{small}
\vskip 0.05in
\raggedright
\footnotesize{$^*$Without indication of deterministic (Det.), the oracle is by default assumed to be stochastic with noise.}
\end{table*}

\textbf{Summary of Contributions.} Our primary contribution is breaking the theoretical bottleneck that historically trapped $1/e$-approximation algorithms in $\mathcal{O}(T^{2/3})$ regret. Specifically:
\begin{enumerate}[leftmargin=*,noitemsep, topsep=0pt]
\item Structural Breakthrough (Theorem~\ref{thm:main}): We prove that non-monotone DR-submodular functions over down-closed convex sets are $1/e$-linearizable. This result mathematically decouples the best-known online approximation ratio from Frank-Wolfe mechanics, enabling a clean reduction to Online Linear Optimization (OLO).
\item Algorithmic Enabler (BQND, Algorithm~\ref{alg:bqnd}): We introduce a novel Jacobian-corrected gradient estimator that constructs unbiased linear surrogates of the non-monotone objective using a strict single-query budget, bypassing the computational cost of standard continuous greedy approximations.
\item Best-known \& Non-Stationary Regret Guarantees (Table~\ref{tab:comparison_master}): By seamlessly plugging our structural theorem into existing OLO meta-algorithms, a broad suite of state-of-the-art guarantees naturally follows as direct corollaries. We obtain best-known $\mathcal{O}(T^{1/2})$ static regret with $\mathcal{O}(1)$ queries, and unlock the first adaptive and dynamic regret guarantees for this setting, alongside improved rates across limited feedback settings.
\end{enumerate}

\section{Related Works}

\textbf{Online Non-monotone DR-Submodular Maximization.} While the online maximization of monotone DR-submodular functions \citep{zhang2022stochastic, hassani2017gradient, chen2018online, fazel2023fast, pedramfar2023unified, pedramfar2025uniform, lu2025decentralized} is well-understood, permitting efficient greedy solutions \citep{streeter2008online}, the non-monotone regime presents significantly greater challenges. Algorithms must balance standard greedy ascent steps with corrective reduction steps to navigate the non-monotone landscape \citep{bian2017continuous}. Recently, \citet{pedramfar2024linear} proposed the elegant framework of \emph{linearizable} functions, which reduces complex non-convex optimization problems, including monotone DR-submodular maximization \citep{pedramfar2024linear}, non-monotone DR-submodular maximization under general convex set \citep{pedramfar2024linear}, regularized phase retrieval problem \citep{sarkar2025online}, one-sided smooth function optimization \cite{pedramfar2026gamma},  to Online Linear Optimization (OLO). 
As a special case, these approaches yield $1/4$ approximation ratio for non-monotone DR-submodular functions under general convex constraint set containing the origin, which is the optimal approximation coefficient for this class of problems. (See \citet{mualem2023resolving}).

\textbf{The Down-Closed Convex Sets \& The Frank-Wolfe Bottleneck.} Overcoming this $1/4$ limit requires focusing on specific constraint geometries, such as down-closed convex sets. Even in the offline setting, finding the optimal approximation ratio over down-closed sets has remained an open challenge for over a decade. While an absolute upper bound of $0.478$ exists \citep{gharan2011submodular, qi2024maximizing}, the best achievable offline rate is currently $0.401$ \citep{buchbinder2024constrained}. 
In the online setting, the best known achievable approximation is $1/e$ \citep{thang2021online, zhang2023online}. 
Historically, achieving this $1/e$ ratio online has been strictly bottlenecked by a reliance on Frank-Wolfe (FW) and continuous greedy algorithms.
This reliance introduces a severe trade-off: FW mechanics couple the non-convexity to the update rules, trapping the regret at $\mathcal{O}(T^{2/3})$ for single-query algorithms \citep{pedramfar2024unified} or requiring computationally expensive batch queries of $\mathcal{O}(T^{3/2})$ to achieve $\mathcal{O}(T^{1/2})$ regret \citep{zhang2023online}. 
This Frank-Wolfe monopoly prevents the use of off-the-shelf OLO/OCO algorithms, and our work breaks this exact bottleneck by proving that $1/e$ can be achieved while preserving a clean linearizable reduction to OCO. Our algorithm is the first to simultaneously achieve $\CO(T^{1/2})$ $1/e$-regret and $\CO(1)$ query efficiency in the down-closed regime.

\textbf{Non-Stationary and Limited Feedback}
While dynamic \citep{zinkevich2003online, zhang2018adaptive, zhao2021bandit} and adaptive \citep{hazan2009efficient, garber2022new} regret bounds are well-established for convex optimization, these guarantees remain less explored for non-monotone DR-submodular functions. The combination of non-convexity and environmental drift presents a formidable barrier that standard expert-tracking meta-algorithms fail to address. 
Recent unified projection-free frameworks developed by \citet{pedramfar2024unified} have successfully addressed various feedback models (semi-bandit, bandit) for adversarial DR-submodular optimization, but they stop short of providing adaptive and dynamic regret guarantees for the challenging non-monotone, down-closed setting.
We provide the first dynamic regret guarantees ($\tilde{O}(\sqrt{T(1+P_T)})$) for this problem class. Furthermore, we extend these robust guarantees to semi-bandit and bandit feedback settings, offering better results for non-stationary maximization under limited information of non-monotone DR-submodular functions over down-closed convex set.

\section{Preliminaries}
\label{sec:preliminaries}

\subsection{Notations}
We use $\C{F}$ to denote the function class to which all objective functions $f_t$ belong.
We use $\C{A}$ to denote an online optimization algorithm.
We denote vectors as boldface lower-case letters (e.g., $\bx,\by \in \mathbb{R}^d$) and denote their coordinates as $x_i, y_i$. 
For any two vectors $\mathbf{x}, \mathbf{y} \in \mathbb{R}^d$, we denote their element-wise product by $\mathbf{x} \odot \mathbf{y}$ and their element-wise exponential by $e^{\mathbf{x}}$. The inequality $\mathbf{x} \le \mathbf{y}$ is understood coordinate-wise.
We define the coordinate-wise probabilistic sum as $\bx \oplus\by \triangleq \bone-(\bone-\bx)\odot(\bone-\by)$.

\textbf{Constraint Set} \quad We consider optimization over a convex set $\mathcal{K} \subseteq [0,1]^d$. We assume $\K$ is \textit{down-closed}, meaning that for any $\mathbf{y} \in \mathcal{K}$ and any $\mathbf{0} \le \mathbf{x} \le \mathbf{y}$, we have $\mathbf{x} \in \mathcal{K}$. 
Additionally, we assume that:
\begin{assumption}[Geometry of Constraint Set]
\label{ass:set}
The constraint set $\mathcal{K}$ has bounded diameter $D > 0$, i.e., $\max_{\mathbf{x}, \mathbf{y} \in \mathcal{K}} \|\mathbf{x} - \mathbf{y}\| \le D$.
\end{assumption}
To ensure computational efficiency, we assume access to a base online linear optimization algorithm $\mathcal{A}$ that is efficient for the constraint set $\mathcal{K}$ using either projection-based  or projection-free (e.g., Frank-Wolfe, SO-OGA) updates.

\textbf{DR-Submodularity} \quad We focus on non-negative, differentiable functions $f: [0,1]^d \to \mathbb{R}_{\ge 0}$. A function $f$ over $\K$ is called \textit{continuous DR-submodular} if $\forall \mathbf{x} \le \mathbf{y} \in \K$, we have $\nabla f(\mathbf{x}) \ge \nabla f(\mathbf{y})$. Equivalently, if $f$ is twice-differentiable, all entries of its Hessian matrix are non-positive ($\nabla^2 f(\mathbf{x}) \le \mathbf{0}$). We explicitly \textit{do not} assume monotonicity; entries of $\nabla f(\mathbf{x})$ may be negative. Additionally, we assume that:
\begin{assumption}[Regularity of Objective Function]
\label{ass:function}
    The functions $f_t \in \C{F}$ are $M_1$-Lipschitz continuous, i.e., $\|\nabla f_t(\mathbf{x})\| \le M_1$ for all $\mathbf{x} \in \mathcal{K}$, and $L$-smooth, i.e., $\|\nabla f_t(\mathbf{x}) - \nabla f_t(\mathbf{y})\| \le L \|\mathbf{x} - \mathbf{y}\|$.
\end{assumption}

\subsection{Problem Setting: Adversarial Online Optimization}
We consider a standard adversarial online optimization of time horizon $T$ with function class $\C{F}$ over constraint set $\K$. 
At round $t$, the player selects a pair of points $\hbx_t, \bu_t \in \mathcal{K}$, an adversary reveals a objective function $f_t \in \C{F}$ with its query oracle $\BF{O}_t$. Then the player plays $\hbx_t$, queries $\BF{O}_t$ at $\bu_t$, and receives feedback $\oo_t$ and updates its decision.

The query oracle is the only way for the player to learn about the functions selected by the adversary. In our work, for the main algorithm, we assume that:
\begin{assumption}[Unbiased Gradient Oracle]
\label{ass:oracle}
For the functions $f_t$, the algorithm has access to a stochastic gradient oracle $\BF{O}$ such that for any query point $\mathbf{x}_t$, it returns a vector $\mathbf{g}_t$ satisfying:
\begin{equation*}
    \E[\mathbf{g}_t \mid \mathbf{x}_t] = \nabla f_t(\mathbf{x}_t) \quad \text{and} \quad \|\mathbf{g}_t\|  \le B_{1}.
\end{equation*}
\end{assumption}

\subsection{Limited Feedback Setting}

When $\bu_t=\hbx_t$, we say the queries are trivial, as the points of query are the same as the played actions; otherwise, we say the queries are non-trivial. We say the provided oracle $\BO_t$ is first-order if it returns the gradient of the given function at the point of query, or zeroth-order if it returns the value of the function.

When the player have non-trivial queries, we say the player takes full-information feedback, which can be either first-order or zeroth-order. 
When the player has trivial queries, we say the player takes semi-bandit feedback if the adversarial provide first-order oracles, or bandit feedback if zeroth-order oracles are provided. 

In Section~\ref{sec:regret}, we begin our analysis with algorithm under first-order full-information feedback, and gives regret analysis. Then we consider zeroth-order full-information feedback, semi-bandit feedback, and bandit feedback,

\subsection{Regret}

Our goal is to minimize the $\alpha$-regret, which measures the gap between a algorithm and an $\alpha$-approximate static optimum. $\alpha$ is often referred to as the optimal approximation ratio. The  $\alpha${\bf-static regret} is defined as:
\begin{equation}
    \mathcal{R}_\alpha(T) \triangleq \alpha \max_{\mathbf{x} \in \mathcal{K}}\sum_{t=1}^T f_t(\mathbf{x}) - \sum_{t=1}^T \E[f_t(\mathbf{x}_t)].
\end{equation}
Beyond static regret, we also consider robustness to non-stationary environments via two advanced metrics. \textbf{Adaptive regret} is defined as the maximum regret over any contiguous interval $[s, e] \subseteq [T]$, i.e.,
\begin{equation*}
    \mathcal{AR}_\alpha(T) \triangleq \sup_{[s,e] \subseteq [T]} \left\{ \alpha \max_{\mathbf{x} \in \mathcal{K}}  \sum_{t=s}^ef_t(\mathbf{x}) - \sum_{t=s}^e  \E[f_t(\mathbf{x}_t)] \right \}.
\end{equation*}
\textbf{Dynamic regret} is defined as the regret against a sequence of time-varying comparators 
$\bu_t^* \triangleq \arg\max_{\bu\in\K} f_t(\bu),$
bounded by the path length $P_T = \sum_{t=1}^T \|\mathbf{u}_t^* - \mathbf{u}_{t-1}^*\|$, i.e.,
\begin{equation}
    \mathcal{DR}_\alpha(T) \triangleq \alpha \sum_{t=1}^T  f_t(\bu_t^*) - \sum_{t=1}^T \E[f_t(\mathbf{x}_t)].
\end{equation}
\begin{remark}[Optimal Approximation Ratio $\alpha$]
Maximizing non-monotone DR-submodular functions is NP-hard, even in the offline setting. Formally, this means the best efficient algorithm can only guarantee a value of $\alpha$ times the optimal value; thus, $\alpha$ is called the optimal approximation ratio. For monotone objectives, algorithms can achieve a stronger approximation ratio of $1-1/e$ (if the constraint set contains the origin) or $1/2$ (for general convex sets). However, these monotonicity-based benchmarks are fundamentally unachievable in the non-monotone regime. For non-monotone functions, no polynomial-time algorithm can achieve an approximation ratio better than $0.478$ \citep{gharan2011submodular, qi2024maximizing}. While recent offline algorithms have narrowed this gap by achieving a $0.401$-approximation \citep{buchbinder2024constrained}, the best known guarantee for efficient online algorithms remains $1/e \approx 0.367$ \citep{thang2021online, zhang2023online}.
\end{remark}

\subsection{Linearizability}
A core tool in our analysis is the \textit{linearizable} framework introduced by \citet{pedramfar2024linear}. This property allows us to reduce non-convex DR-submodular optimization to online linear optimization.

\begin{definition}[Linearizability]
\label{def:upper_linearizable}
A function $f$ over $\mathcal{K}$ is $\alpha$-linearizable if there exists a mapping $h: \mathcal{K} \to \mathcal{K}$ and a vector field $\mathfrak{g}: \mathcal{F} \times \mathcal{K} \to \mathbb{R}^d$ such that for all $\mathbf{x}, \mathbf{y} \in \mathcal{K}$:
\begin{equation}
    \alpha f(\mathbf{y}) - f(h(\mathbf{x})) \le \beta\langle \mathfrak{g}(f, \mathbf{x}), \mathbf{y} - \mathbf{x} \rangle.
\end{equation}
\end{definition}
Intuitively, the definition reduces online non-convex DR-submodular maximization to online linear optimization (OLO): it upper-bounds the scaled reward $\alpha f(\by)$ (up to $f(h(\bx))$) by the linear term $\langle \mathfrak{g}(f,\bx),\,\by-\bx\rangle$.
Thus, any low-regret OLO algorithm run on $\mathfrak{g}(f_t,\bx_t)$ transfers to $\alpha$-regret for the original rewards, up to the constant $\beta$ (and the mapping $h$).
We call $h$ reparameterization, $\beta$ scaling parameter, and $\gfrak$ the surrogate potential.
Here $\alpha$ is exactly the approximation/competitive factor in our regret benchmark.

\section{Non-monotone DR-submodular Functions over Down-Closed Convex Sets are $1/e$-Linearizable}
\label{sec:linearizability}

As previously described, in this paper, we consider the case where the constraint set $\mathcal{K} \subseteq [0,1]^d$ is a down-closed convex set containing the origin. We show that the class of non-monotone DR-submodular functions over such sets is linearizable with $1/e$-approximation ratio.

A central challenge for proving upper-linearizablility of non-monotone functions over $\K$ is to construct a global lower bound that relates the algorithm's trajectory to an arbitrary comparator $\mathbf{y}$ (e.g., the global optimum). Unlike monotone regimes where simple greedy algorithms suffice, the non-monotone landscape requires balance between ascending gradients and shrinking constraints. 

The following key Lemma overcomes this by establishing a novel structural inequality for our exponential reparameterization $h_z(\mathbf{x})$. 
This inequality is a mathematical specialization of Lemma 4.1 from \citet{buchbinder2024constrained} to the constant-path case.
By leveraging the down-closed property of the constraint set, this formulation provides the exact structural piece required to accommodate the adversarial online setting.
It proves that this specific reparameterization preserves a strict lower bound relative to any target $\mathbf{y}$, effectively bridging the gap between the DR-submodular property and our $1/e$ approximation ratio.

\begin{lemma}
\label{lem:niv}
Let $f$ be a non-monotone continuous DR-submodular function over a down-closed convex set.
For any $\bx, \by \in [0,1]^d$ and $z \in [0,1]$, let $h_z(\bx)=\bone - e^{-z\bx}$, we have:
\begin{equation} \label{eq:lemma41}
    f(h_z(\bx) \oplus \by) \ge e^{-z\bar{x}} f(\by) \geq e^{-z} f(\by)
\end{equation}
where $\bar{x} = \max_j [\bx]_j$. (here $[\bx]_j$ is the $j$-th component of the vector $\bx$).
\end{lemma}

The proof of Lemma~\ref{lem:niv} is provided in Appendix~\ref{app:lem_niv_proof}. Next, we have the main result of this paper:

\begin{theorem}[Main Theorem]\label{thm:main}
Let $\mathcal{K} \subseteq [0,1]^d$ be a down-closed convex set such that $\bzero \in \mathcal{K}$.  
Let $f: [0,1]^d \to \R_{\ge 0}$ be a non-negative, differentiable, non-monotone DR-submodular function.
Define the mapping $h: \mathcal{K} \to \mathcal{K}$ as $h(\bx) \triangleq \bone - e^{-\bx}$ and $h_z(\bx) \triangleq \bone - e^{-z\bx}$.
Let the vector field $\gfrak: \Fcal \times \mathcal{K} \to \R^d$ be $\gfrak(f, \bx) \triangleq \grad F(\bx)$, where $F: \mathcal{K} \to \R$ is the function defined by:
\begin{equation}
\label{eq:F}
    F(\bx) \triangleq \int_{0}^{1} \frac{e^{z-1}}{(1 - e^{-1})z} \left( f(\bone - e^{-z\bx}) - f(\mathbf{0})\right) \, dz.
\end{equation}
Then for $\forall\bx, \by \in \mathcal{K}$, the following inequality holds:
\begin{equation}
\frac{1}{e} f(\by) - f(h(\bx))
\le
    (1-e^{-1})\langle \gfrak(f, \bx), \by - \bx \rangle.
\end{equation}
Thus, the function $f$ is $\alpha$-linearizable with approximation coefficient $\alpha = \frac{1}{e}$, the scaling parameter $\beta=1-e^{-1}$, the reparamterization $h(\bx)=1-e^{-\bx}$, and surrogate potential $\gfrak$.
\end{theorem}

\begin{remark}
    Intuitively, the exponential mapping $h(\mathbf{x})$ is specifically chosen because it ensures that the reparameterized trajectory strictly remains within the valid domain for downward-closed convex sets. In the same time, $F(\mathbf{x})$ acts as an integrating factor that heavily weights the early stages of the trajectory, which enables an exact mathematical cancellation of the non-monotone penalty during the integration-by-parts step of the subsequent analysis.
\end{remark}

\begin{proof}
Clearly we have $F(\BF{0})=0$. For any $\bx \neq \BF{0}$, the integrand in Equation~\eqref{eq:F} is a continuous function of $z$ that is bounded by:
\begin{align*}
    &\frac{e^{z-1}}{(1 - e^{-1})z} \left( f(\bone - e^{-z\bx}) - f(\mathbf{0})\right) \\
    &\qquad \overset{(a)}{\le} \frac{e^{z-1}}{(1 - e^{-1})z} M_1 \Vert \bone - e^{-z\bx} \Vert \overset{(b)}{\le} \frac{1}{1 - e^{-1}} M_1,
\end{align*}
where (a) follows from the $M_1$-Lipschitz continuity of $f$, and (b) uses the bound
$\|\bone - e^{-z\bx}\|\le 1$.
Therefore $F$ is well-defined on $[0,1]^d$.

Next, differentiating $F$ with respect to $\bx$,
\begin{align}
\nabla F(\bx)
&= \int_0^1 \frac{e^{z-1}}{(1-e^{-1})z}\,\nabla\!\big(f(h_z(\bx))-f(\bzero)\big)\,dz \nonumber\\
&\overset{(a)}{=} \int_0^1 \frac{e^{z-1}}{(1-e^{-1})z}\,\Big(\nabla f(h_z(\bx))\odot \frac{\partial}{\partial \bx}h_z(\bx)\Big)\,dz \nonumber\\
&\overset{(b)}{=} \int_0^1 \frac{e^{z-1}}{1-e^{-1}}\,\big(\nabla f(h_z(\bx))\odot e^{-z\bx}\big)\,dz. \label{eq:gradF_clean}
\end{align}
where (a) uses the chain rule w.r.t. $\bx$, and (b) uses $\frac{\partial}{\partial \bx} h_z(\bx)= z\, e^{-z\bx}$ because $h_z(\bx)=\bone-e^{-z\bx}$.

Now, we use \eqref{eq:gradF_clean}, and $\bx$ is independent of $z$, we have
\begin{align}
&(1-e^{-1})\langle \nabla F(\bx),-\bx\rangle \nonumber\\
&\qquad\overset{(a)}{=} \int_0^1 e^{z-1}\left\langle \nabla f(h_z(\bx))\odot e^{-z\bx},\; -\bx \right\rangle dz \nonumber\\
&\qquad\overset{(b)}{=} \int_0^1 e^{z-1}\left\langle \nabla f(h_z(\bx)),\; -\bx\odot e^{-z\bx} \right\rangle dz, \nonumber \\
&\qquad\overset{(c)}{=} \int_0^1 e^{z-1}\left(-\frac{d}{dz} f(h_z(\bx))\right)\,dz \label{eq:cost_inter}
\end{align}
where (a) uses linearity of the inner product and interchange of integral and inner product, (b) uses the identity
$\langle \mathbf{a}\odot \mathbf{b},\,\mathbf{c}\rangle=\langle \mathbf{a},\,\mathbf{b}\odot \mathbf{c}\rangle$, and (c) uses the chain rule in $z$ that $\frac{d}{dz} f(h_z(\bx))
=\left\langle \nabla f(h_z(\bx)),\,\frac{d}{dz}h_z(\bx)\right\rangle 
=\left\langle \nabla f(h_z(\bx)),\,\bx\odot e^{-z\bx}\right\rangle$.

We now apply integration by parts to \eqref{eq:cost_inter} with $u(z)=e^{z-1}$ and $dv(z)= -\frac{d}{dz}f(h_z(\bx))\,dz$.
Then $du(z)=e^{z-1}dz$ and $v(z)=-f(h_z(\bx))$. Hence,
\begin{align*}
&(1-e^{-1})\langle \nabla F(\bx),-\bx\rangle \\
&\qquad=
\Big[-e^{z-1}f(h_z(\bx))\Big]_{0}^{1}
+\int_0^1 e^{z-1}f(h_z(\bx))\,dz
\\
&\qquad=
-f(h_1(\bx))+\frac{1}{e}f(h_0(\bx))
+\int_0^1 e^{z-1}f(h_z(\bx))\,dz \\
&\qquad\ge 
-f(h(\bx))+\int_0^1 e^{z-1}f(h_z(\bx))\,dz,
\end{align*}
since $h_1(\bx)=h(\bx)$ and $h_0(\bx)=\bzero$.

Again, using \eqref{eq:gradF_clean}, we have
\begin{align*}
    &(1-e^{-1})\langle \grad F(\bx), \by \rangle \\
    &\qquad= \int_{0}^{1} e^{z-1} \langle \grad f(h_z(\bx)) \odot e^{-z\bx}, \by \rangle \, dz \\
    &\qquad\overset{(a)}{=} \int_{0}^{1} e^{z-1} \langle \grad f(h_z(\bx)), \by \odot e^{-z\bx} \rangle \, dz \\
    &\qquad\overset{(b)}{=} \int_{0}^{1} e^{z-1} \langle \grad f(h_z(\bx)), \by \odot (1-h_z(\bx)) \rangle \, dz \\
    &\qquad\overset{(c)}{\ge} \int_{0}^{1} e^{z-1} \left[ f(h_z(\bx) \oplus \by) - f(h_z(\bx)) \right] \, dz
\end{align*}
where (a) uses the identity
$\langle \mathbf{a}\odot \mathbf{b},\,\mathbf{c}\rangle=\langle \mathbf{a},\,\mathbf{b}\odot \mathbf{c}\rangle$, (b) is due to $h_z(\bx) = \bone - e^{-z\bx}$, and (c) uses Lemma~\ref{lem:niv_sub} and the fact that $h_z(\bx) \oplus \by = \bone-(\bone-h_z(\bx))\odot(\bone-\by) = h_z(\bx) + \by \odot (1-h_z(\bx))$.

Lemma~\ref{lem:niv} states that, $\forall \bx, \by \in [0,1]^d$ and $z\in [0,1]$, we have $f(h_z(\bx) \oplus \by) \ge e^{-z} f(\by)$. Thus,

\begin{align}
    &(1-e^{-1})\langle \grad \Fcal(\bx), \by \rangle \nonumber\\
    &\qquad\ge \int_{0}^{1} e^{z-1} \left[ e^{-z} f(\by) - f(h_z(\bx)) \right] \, dz \nonumber\\
    &\qquad\ge f(\by) \int_{0}^{1} e^{z-1} e^{-z} \, dz - \int_{0}^{1} e^{z-1} f(h_z(\bx)) \, dz \nonumber \\
    &\qquad= f(\by) \int_{0}^{1} e^{-1} \, dz - \int_{0}^{1} e^{z-1} f(h_z(\bx)) \, dz \nonumber \\
    &\qquad= \frac{1}{e} f(\by) - \int_{0}^{1} e^{z-1} f(h_z(\bx)) \, dz. \label{}
\end{align}

Summing both terms:
\begin{align*}
    &(1-e^{-1})\langle \gfrak(f,\bx), \by - \bx \rangle \\
    &\qquad= (1-e^{-1})\langle \grad F(\bx), \by - \bx \rangle \\
    &\qquad = (1-e^{-1})\left[ \langle \grad \Fcal(\bx), \by \rangle + \langle \grad \Fcal(\bx), - \bx \rangle\right] \\
    &\qquad = \frac{1}{e} f(\by) - f(h(\bx))
\end{align*}
because the integral terms $\int_{0}^{1} e^{z-1} f(h_z) \, dz$ cancels.    
\end{proof}

\begin{remark}[Expectation form of $\nabla F$]
\label{rem:gradF_expect}
Let $\mathcal{Z}\in[0,1]$ be a random variable with CDF
\begin{equation}
\label{eq:sample}
    \B{P}(\mathcal{Z}\le z)=\int_0^z \frac{e^{u-1}}{1-e^{-1}}\,du,
\end{equation}
equivalently with density $p(z)=\frac{e^{z-1}}{1-e^{-1}}$ on $[0,1]$. Then \eqref{eq:gradF_clean} admits the expectation representation
\begin{equation}\label{eq:gradF_expect}
\nabla F(\bx)=\E_{\mathcal{Z}}\big[\nabla f(h_{\mathcal{Z}}(\bx))\odot e^{-\mathcal{Z}\bx}\big].
\end{equation}
\end{remark}

In order to attain an unbiased and bounded estimate of $\gfrak$, we provide Boosted Query algorithm for Non-monotone DR-submodular functions over Down-closed convex set (BQND), detailed in Algorithm~\ref{alg:bqnd}, and we formally state these properties in Lemma~\ref{lem:bqnd}.
Algorithm~\ref{alg:bqnd} is necessary to provide regret guarantee for our main algorithm, as it satisfies the conditions of Theorem~\ref{thm:regret_transfer}, the Regret Transfer Theorem. Theorem~\ref{thm:regret_transfer} establishes the fundamental regret transfer guarantee for our framework, proving that the difficult problem of non-monotone maximization strictly reduces to the simpler problem of online linear optimization. It formalizes the linearizable reduction, guaranteeing that the regret of our main algorithm is bounded by the regret of the linear base learner on the surrogate functions, up to a scaling constant $\beta$. This allows us to directly inherit the convergence rates of the chosen base solver.

\begin{algorithm}[h]
\caption{Boosted Query algorithm for Non-monotone DR-Submodular functions over Down-closed convex set (BQND)}
\label{alg:bqnd}
\begin{algorithmic}[1]
\STATE \textbf{Input:} Point $\mathbf{x} \in \mathcal{K}$, first-order stochastic oracle $\BF{O}$ for $f$. 
\STATE \textbf{Sampling:} Sample $z\in [0,1]$ according to Equation~\ref{eq:sample}
\STATE \textbf{Query:} Construct point of query $\mathbf{u} = \mathbf{1} - e^{-z\mathbf{x}}$.
\STATE Query the first-order oracle $\BF{O}$ which returns a stochastic sample of $\nabla f(\bu)$, denoted as $\mathbf{v}$.
\STATE \textbf{Output:} Return the estimator $\mathbf{g} =   \mathbf{v} \odot e^{-z\mathbf{x}}$.
\end{algorithmic}
\end{algorithm}

\begin{lemma}[Properties of BQND Estimator]
\label{lem:bqnd}
Let $\mathbf{g}$ be the output of Algorithm \ref{alg:bqnd} for an input $\mathbf{x} \in \mathcal{K}$ and a first-order oracle for function $f$ that satisfies Assumption~\ref{ass:oracle}. Then $\mathbf{g}$ satisfies that:
\[
\E[\mathbf{g}] = \gfrak(f,\bx) = \nabla F(\bx), \text{ and } \|\mathbf{g}\| \le B_1.
\]
i.e., Algorithm~\ref{alg:bqnd} returns an unbiased estimate of $\gfrak$ that is bounded by $B_1$.
\end{lemma}

\begin{proof}
From Algorithm \ref{alg:bqnd}, the output is $\mathbf{g} = \bv \odot e^{-z\mathbf{x}}$, where under Assumption~\ref{ass:oracle}, $\mathbb{E}[\mathbf{v}] =\nabla f( \mathbf{1} - e^{-z\mathbf{x}})$ and $z$ is sampled  according to Equation~\ref{eq:sample}. Taking the expectation over $z$:
\begin{align*}
    \mathbb{E}[\mathbf{g}] = \int_0^1 \frac{e^{z-1}}{1-e^{-1}} \left( e^{-z\mathbf{x}} \odot \nabla f(\mathbf{1} - e^{-z\mathbf{x}}) \right) dz.
\end{align*}
This integral is identical to the gradient derivation of the surrogate function $F(\mathbf{x})$ proved in Theorem \ref{thm:main}. Thus, $\mathbb{E}[\mathbf{g}] = \nabla F(\mathbf{x})=\gfrak(\bx)$.

The norm of the output is:
\vspace{1pt}
\begin{align*}
    \|\mathbf{g}\| &= \|\bv \odot e^{-z\mathbf{x}}\| 
    \le \|\bv\| \cdot \max_{i} |e^{-z x_i}|
\end{align*}
Since $\mathbf{x} \in \mathcal{K} \subseteq [0, 1]^d$ and $z \ge 0$, the term $0 < e^{-z x_i} \le 1$ for all $i$. Therefore, the element-wise shrinking cannot increase the norm. 
Using Assumption~\ref{ass:oracle}, we have
\begin{equation*}
\|\mathbf{g}\| \le B_1 \cdot 1 = B_1.
\qedhere
\end{equation*}
\end{proof}

\section{Regret Results Applied from Linearizable Optimization}
\label{sec:regret}

A key advantage of the linearizable formulation is that it allows us to seamlessly transfer guarantees from Online Linear/Convex Optimization to our non-monotone DR-submodular setting. To begin with, we consider the most common setting, where the adversary provides a first-order noisy oracle as described in Assumption~\ref{ass:oracle}. We propose Algorithm~\ref{alg:main}, a modular reduction framework for maximizing non-monotone DR-submodular functions over down-closed convex sets.

The algorithm requires initializing two subroutines: 
a base learner $\mathcal{A}$, which can be any efficient regret-minimizing algorithm for online linear optimization over $\mathcal{K}$, 
and a query algorithm $\mathcal{G}$ that is fixed to be BQND (Algorithm~\ref{alg:bqnd}). 
While our framework is compatible with any such linear solver, to obtain the specific regret guarantees in this paper, we instantiate $\mathcal{A}$ with Online Gradient Ascent via Separation Oracle (SO-OGA) (Appendix~\ref{app:so-oga}) or Improved Ader (IA) (Appendix~\ref{app:dynamic_algos}). 
The base learner receives first-order feedback and updates its decision to $\mathbf{x}_t$. 
Our main algorithm then plays the action $\hat{\mathbf{x}}_t$ and passes $\mathbf{x}_t$ to the query algorithm, which interacts with the oracle to return an unbiased estimate of $\mathfrak{g}$.

\begin{algorithm}
\caption{Adaptive Projection-free Online Non-monotone DR-Submodular Maximization over Down-closed Convex Sets}
\label{alg:main}
\begin{algorithmic}[1]
\STATE \textbf{Input:} Horizon $T$, Constraint set $\mathcal{K}$, Stepsize $\eta$.
\STATE \textbf{Initialize:} 
\STATE \quad Base Learner $\mathcal{A}$: Any online linear optimization algorithm with sublinear regret.\footnotemark
\STATE \quad Query Algorithm $\mathcal{G}$: \textsc{BQND} (Algorithm~\ref{alg:bqnd}). 
\FOR{$t = 1, \dots, T$}
    \STATE Receive action $\mathbf{x}_t$ from Base Learner $\mathcal{A}$
    \STATE Play action $\hat{\bx}_t = \mathbf{1} - e^{-\mathbf{x}_t}$
    \STATE Adversary selects a function $f_t$ and a first-order query oracle $\BF{O}_t$.
    \STATE Run query algorithm: $\BF{o}_t \leftarrow \mathcal{G}(\BF{O}_t, \mathbf{x}_t)$.
    \STATE Pass $\BF{o}_t$ to $\mathcal{A}$ to update its state.
\ENDFOR
\end{algorithmic}
\end{algorithm}
\footnotetext{Specific instantiations of $\C{A}$ for our results are given in the propositions.}

Algorithm \ref{alg:main} follows the reduction-to-linear-optimization paradigm for linearizable functions, instantiating the Online Maximization By Quadratization (OMBQ) meta-algorithm proposed by \citet{pedramfar2024linear} (see Appendix~\ref{app:ombq}) with our specific exponential mapping $h(\mathbf{x}) = \mathbf{1} - e^{-\mathbf{x}}$ and the BQND query algorithm. 
Because Theorem~\ref{thm:main} successfully establishes the structural $1/e$-linearizability of our domain, the subsequent regret guarantees across various feedback models do not require independent non-convex analyses. Instead, the bounds presented in the following subsections and summarized in Table~\ref{tab:comparison_master} follow naturally as direct mechanical corollaries of plugging Theorem~\ref{thm:main} into the OMBQ framework alongside appropriate linear base learners.
This configuration yields the first algorithm for this domain with improved regret permitting $\mathcal{O}(1)$ oracle queries per round.

\subsection{Static Regret}

Our framework (Algorithm~\ref{alg:main}) works with any base learner $\mathcal{A}$ that minimizes regret for linear functions. To derive improved regret rate that is efficient, we instantiate $\mathcal{A}$ with the projection-free learner SO-OGA in Appendix~\ref{app:so-oga}. This algorithm is proposed by \citet{garber2022new} and later refined by \citet{pedramfar2024linear}. The result is as follows:

\begin{proposition}[Static Regret]
\label{prop:static-regret}
Let $\{f_t\}_{t=1}^T$ be a sequence of $M_1$-Lipschitz non-monotone DR-submodular functions. Instantiating base learner to be SO-OGA as described in Algorithm~\ref{alg:so_oga_base}, Algorithm \ref{alg:main} achieves a $1/e$-approximation static regret bounded by:
\begin{equation}
    \E\left[ \alpha \max_{\bu \in \mathcal{K}}\sum_{t=1}^T f_t(\bu) - \sum_{t=1}^T f_t(\hbx_t) \right] 
    \le \CO(M_1 \sqrt{T})
\end{equation}
for any fixed comparator $\by \in \K$. The algorithm requires only $\CO(1)$ gradient queries per round.
\end{proposition}

\begin{proof}
Since we choose the base learner to be SO-OGA, it follows from Theorem~\ref{thm:so_oga_regret} that 
\begin{align*}
    \mathcal{R}_{1}^{\text{SO-OGA}} = \CO(M_1 T^{1/2})
\end{align*}
In Lemma~\ref{lem:bqnd}, we proved that Algorithm~\ref{alg:bqnd} returns an unbiased bounded estimate of $\gfrak$. Thus, using Theorem~\ref{thm:regret_transfer}, we have:
\begin{align*}
    \E\left[ \alpha \max_{\bu \in \mathcal{K}}\sum_{t=1}^T f_t(\bu) - \sum_{t=1}^T f_t(\hbx_t) \right] 
    &\le \beta \mathcal{R}_{1}^{\text{SO-OGA}} \\
    &= \CO(M_1 T^{1/2})
    \qedhere
\end{align*}
\end{proof}

\subsection{Adaptive and Dynamic Regrets}

Since \textsc{SO-OGA} is known to minimize adaptive regret for linear functions \citep{garber2022new,pedramfar2024linear}, Algorithm~\ref{alg:main} instance in Proposition~\ref{prop:static-regret} automatically achieves an adaptive regret bound described below:
\begin{proposition}[Adaptive Regret]
\label{prop:adaptive-regret}
Let $\{f_t\}_{t=1}^T$ be a sequence of $M_1$-Lipschitz non-monotone DR-submodular functions over down-closed convex sets.
Instantiating base learner to be SO-OGA as described in Algorithm~\ref{alg:so_oga_base}, Algorithm~\ref{alg:main} achieves a $1/e$-approximation adaptive regret bounded by:
\begin{equation}
    \mathcal{AR}_{1/e}(T) \le \CO(T^{1/2})
\end{equation}
This is the first adaptive regret guarantee for non-monotone DR-submodular maximization over down-closed sets.
\end{proposition}
\begin{proof}
If we choose the base learner to be SO-OGA, it follows from Theorem~\ref{thm:so_oga_regret} that 
\begin{align*}
    \mathcal{AR}_{1}^{\text{SO-OGA}} = \CO(M_1 T^{1/2})
\end{align*}
In Lemma~\ref{lem:bqnd}, we proved that Algorithm~\ref{alg:bqnd} returns an unbiased bounded estimate of $\gfrak$. Thus, using Theorem~\ref{thm:regret_transfer}, we have:
\begin{align*}
\mathcal{AR}_{1/e}(T)
&\le \beta \mathcal{R}_{1}^{\text{SO-OGA}} 
= \CO(T^{1/2}).
\qedhere
\end{align*}
\end{proof}

To handle non-stationary environments, we instantiate Algorithm~\ref{alg:main} with \textsc{Improved Ader} algorithm (Algorithm~\ref{alg:improved_ader}) in Appendix \ref{app:dynamic_algos} as the base linear learner $\mathcal{A}$. 
This algorithm is proposed by \citet{zhang2018adaptive} and later refined by \citet{pedramfar2024linear}. 
The result is as follows:

\begin{proposition}[Dynamic Regret Guarantees]
\label{prop:dynamic-regret}
Let $P_T = \sum_{t=1}^T \|\mathbf{x}^*_t - \mathbf{x}^*_{t-1}\|$ denote the path length of the sequence of optimal minimizers for the surrogate linear functions. 
Let $\{f_t\}_{t=1}^T$ be a sequence of $M_1$-Lipschitz non-monotone DR-submodular functions over down-closed convex sets.
Instantiating base learner to be IA as described in Algorithm~\ref{alg:improved_ader}, Algorithm~\ref{alg:main} achieves a $1/e$-approximation dynamic regret bounded by:
\begin{equation}
    \mathcal{DR}_{1/e}(T) \le \tilde{O}\left( \sqrt{T(1 + P_T)} \right).
\end{equation}
\end{proposition}
\begin{proof}
If we choose the base learner to be IA for linear functions, it follows from Theorem~\ref{thm:dynamic_regret_bound} that 
\begin{align*}
    \mathcal{R}^{\text{IA}}_{1,\BF{L}}(\mathbf{u}) = O\left(M_1 \sqrt{T(1 + P_T(\mathbf{u}))}\right)
\end{align*}
In Lemma~\ref{lem:bqnd}, we proved that Algorithm~\ref{alg:bqnd} returns an unbiased bounded estimate of $\gfrak$. Thus, using Theorem~\ref{thm:regret_transfer}, we have:
\begin{align*}
    \mathcal{DR}_{1/e}(\bu)
    &\le \beta \mathcal{R}_{1}^{\text{IA}}(\mathbf{u}) 
    = \tilde\CO(\sqrt{T(1 + P_T)})
    \qedhere
\end{align*}
\end{proof}

\subsection{Semi-Bandit, Zeroth-order Full-Information, and Bandit Feedback}
Recall that for Algorithm~\ref{alg:main}, we assumed first-order gradient, and the queries determined by the query algorithm BQND are non-trivial. Thus, the results we obtained in Theorem~\ref{prop:static-regret}, ~\ref{prop:adaptive-regret}, and ~\ref{prop:dynamic-regret} are given for first-order full-information feedback.
By decoupling the non-convex query algorithm from the linear base learner, our modular design allows seamless extensions to restrictive feedback settings.
In this section, we show that by applying meta-algorithms developed for linearizable functions by \citet{pedramfar2024linear}, our regret guarantees naturally transfer to Semi-Bandit, Zeroth-Order Full-Information, and Bandit settings. Furthermore, we provide the first theoretical analysis extending these results to Adaptive and Dynamic regret measures, establishing a unified projection-free framework that is robust to both limited information and non-stationary feedback.

Given first-order oracles, when only trivial queries are permitted, we say that the algorithm handles semi-bandit feedback. Thus, we apply SFTT (Algorithm~\ref{alg:sftt_meta}), with $\C{A}^{action}$ being the SO-OGA and $\C{A}^{query}$ being BQND. Thus, applying Lemma~\ref{lem:sftt_reduction} with $\eta=1/2$ due to Proposition~\ref{prop:static-regret} and \ref{prop:adaptive-regret}, we have the following result:
\begin{proposition}[Semi-Bandit Guarantees]
\label{prop:semi_bandit}
Let $\{f_t\}_{t=1}^T$ be a sequence of non-monotone DR-submodular functions over down-closed convex sets.
If we instantiate \textsc{SFTT} meta-algorithm (Algorithm~\ref{alg:sftt_meta}) with block size $L = T^{1/3}$, using \textsc{SO-OGA} as the base learner $\mathcal{A}^{\text{action}}$ and \textsc{BQND} as the query algorithm $\mathcal{A}^{\text{query}}$, the resulting algorithm achieves the following regret bounds against $\{f_t\}_{t=1}^T$:
\begin{equation*}
\mathbb{E}[\mathcal{R}(T)] \le O(T^{2/3}) \quad \text{and} \quad \mathcal{AR}_{1/e}(T) \le O(T^{2/3}).
\end{equation*}
\end{proposition}

When provided only zeroth-order oracles (i.e., oracles returning value estimates for functions instead of gradient estimates), we say that the algorithm handles zeroth-order full-information feedback if the queries are non-trivial, or bandit feedback if the queries are trivial. For zeroth-order full-information feedback, we apply \textsc{FOTZO} (Algorithm~\ref{alg:fotzo_meta}), with $\C{A}$ being the SO-OGA and $\C{A}^{query}$ being BQND. Thus, applying Lemma~\ref{lem:fotzo_reduction} with $\eta=1/2$ due to Proposition~\ref{prop:static-regret} and \ref{prop:adaptive-regret}, we have the following result:
\begin{proposition}[Zeroth-Order Full-Information Guarantees]
\label{prop:zo_full_info}
Let $\{f_t\}_{t=1}^T$ be a sequence of non-monotone DR-submodular functions over down-closed convex sets.
Let Algorithm $\mathcal{A}$ be the instantiation of the \textsc{FOTZO} (Algorithm 6) using \textsc{BQND} as $\mathcal{A}^{\text{query}}$ equipped with base algorithm \textsc{SO-OGA}. By employing a one-point gradient estimator with smoothing radius $\delta$, the resulting Algorithm~$\mathcal{A}$ achieves the following regret bounds against $\{f_t\}_{t=1}^T$:
\begin{equation*}
\mathbb{E}[\mathcal{R}(T)] \le O(T^{3/4}) \quad \text{and} \quad \mathcal{AR}_{1/e}(T) \le O(T^{3/4}).
\end{equation*}
\end{proposition}

If only a zeroth-order oracle is provided, and only trivial queries are allowed, we say that the algorithm handles bandit feedback. For such limited feedback, we apply \textsc{STB} meta-algorithm to the algorithm instantiated in Proposition~\ref{prop:semi_bandit}, and applying Lemma~\ref{lem:stb_reduction}, we obtain the following results:
\begin{proposition}[Bandit Guarantees]
\label{prop:bandit}
Let $\{f_t\}_{t=1}^T$ be a sequence of non-monotone DR-submodular functions over down-closed convex sets.
Let Algorithm $\mathcal{A}$ be the semi-bandit algorithm instantiated in Proposition~\ref{prop:semi_bandit}. Applying STB (Algorithm~\ref{alg:stb_meta}), the resulting Algorithm achieves the following regret bounds against $\{f_t\}_{t=1}^T$:
\begin{equation*} 
\mathbb{E}[\mathcal{R}(T)] \le O(T^{2/3})
\quad \text{and} \quad \mathcal{AR}_{1/e}(T) \le O(T^{2/3}).
\end{equation*}
\end{proposition}

Note that the reduction lemmas (Lemmas~\ref{lem:sftt_reduction}, \ref{lem:fotzo_reduction}, and \ref{lem:stb_reduction}) transfer the regret guarantees from first-order full-information setting to limited settings regardless of the choice of base learner. Consequently, by instantiating the meta-algorithms with the \textsc{Improved Ader} base learner (Proposition~\ref{prop:dynamic-regret}), we obtain dynamic regret bounds for all limited feedback settings. 
\begin{proposition}[Dynamic Regret for Limited Feedback]\label{prop:dynamic_limited}
The expected dynamic regret $\E[\mathcal{R}_{1/e}^{\text{dyn}}(T)]$ for the described algorithm in the corresponding proposition is bounded by
$\tilde{O}(T^{2/3}\sqrt{1+P_T})$ for Semi-Bandit feedback,
$\tilde{O}(T^{3/4}\sqrt{1+P_T})$ for Zeroth-Order Full-Info feedback,
and $\tilde{O}(T^{5/6}\sqrt{1+P_T})$ for Bandit feedback.
where $P_T$ is the path length of the optimal sequence.
\end{proposition}

\section{Conclusion and Future Work}
\textbf{Conclusions.} In this work, we resolved a critical theoretical bottleneck in online non-monotone DR-submodular maximization over down-closed convex sets by proving that the function class is strictly $1/e$-linearizable. By carefully designing an exponential reparameterization and a surrogate potential that yields an exact cancellation during integration-by-parts, we successfully decoupled the $1/e$ approximation limit from the computationally heavy Frank-Wolfe mechanics. Coupled with our BQND gradient estimator, this structural breakthrough enabled a clean reduction to standard Online Linear Optimization utilizing only a strict single-query budget per round. As a result, we were able to elevate the state of the art across multiple feedback models, providing $\mathcal{O}(T^{1/2})$ static regret alongside the first adaptive and dynamic regret guarantees for this  domain.

\textbf{Limitations and Future Work.} While our linearizable reduction achieves the best-known online approximation ratio $1/e$ with $O(T^{1/2})$ regret, a theoretical gap remains compared to the best achievable offline approximation of $0.401$ \citep{buchbinder2024constrained}. A fundamental limitation in the online setting is the adversarial revelation of the non-monotone penalty, which currently prevents algorithms from making the globally aware corrective steps utilized by state-of-the-art offline methods. A direction for future work is determining whether this $1/e$ online barrier can be broken, perhaps by exploring predictive feedback models or relaxed constraint geometries. Additionally, because the optimal online approximation ratio remains an open question, establishing strict, end-to-end regret lower bounds for the limited-feedback regimes explored in this paper represents another crucial open challenge. Finally, extending the results to $\gamma$-weakly DR-submodular functions, as studied in \cite{pedramfar2024unifiedonline,jadav2026stronger}, remains an open direction.

\section*{Impact Statement}
This paper presents work whose goal is to advance the field of Machine
Learning. There are many potential societal consequences of our work, none
which we feel must be specifically highlighted here.

\bibliography{main}
\bibliographystyle{icml2026}

\newpage
\appendix
\onecolumn

\section{Useful Lemmas}\label{sec:lemma}

For non-negative, continuous, non-monotone DR-submodular functions over down-closed convex sets, \cite{buchbinder2024constrained} established several inequalities that we use in our analysis.

\begin{lemma}[Lemma 2.1, \citet{buchbinder2024constrained}]
\label{lem:niv_sub}
Let $f$ be a non-negative continuous DR-submodular function over $[0,1]^d$. Then, $\forall \bx \in [0,1]^d$ and $\by \ge \bzero$ such that $\bx+\by\le \bone$, we have:
\[
\langle \nabla f(\bx),\by \rangle\ge f(\bx+\by)-f(\bx).
\]
\end{lemma}

\begin{lemma}[Lemma 4.1, \citet{buchbinder2024constrained}]
    Let $f: [0, 1]^d \to \mathbb{R}_{\ge 0}$ be a non-negative DR-submodular function. Given $t \ge 0$, an integrable function $\mathbf{x}: [0, t] \to [0, 1]^d$, and a vector $\mathbf{a} \in [0, 1]^d$, the original lemma states:
\begin{equation}\label{eq:122}
     f\left(\mathbf{1} - \mathbf{a} \odot e^{-\int_0^t \mathbf{x}(\tau)d\tau}\right) \ge e^{-t} \cdot \left[ f(\mathbf{1} - \mathbf{a}) + \sum_{i=1}^{\infty} \frac{1}{i!} \cdot \int_{\tau \in [0,t]^i} f\left((\mathbf{1} - \mathbf{a}) \oplus \bigoplus_{j=1}^i \mathbf{x}(\tau_j)\right) d\tau \right]
\end{equation}
   
\end{lemma}

\section{Proof of Lemma~\ref{lem:niv}}
\label{app:lem_niv_proof}

\begin{proof}
\noindent
In \eqref{eq:122}, for online adversarial setting, we assume $\mathbf{x}(\tau) = \mathbf{x}$ is constant over the interval $[0, t]$, so on the left hand side, $\int_0^t \mathbf{x}(\tau)d\tau = t\mathbf{x}$, and on the right hand side, $\bigoplus_{j=1}^i \mathbf{x}(\tau_j)=1-\odot_{j=1}^i (1-\bx) = 1 - (1-\bx)^i$, and because the function $F$ takes values in $\mathbb{R}_{\ge 0}$ (it is non-negative), every term inside the integral is non-negative. Consequently, the entire infinite sum is non-negative:
\begin{align*}
    \sum_{i=1}^{\infty} \frac{1}{i!} \cdot \int_{\tau \in [0,t]^i} f\left((\mathbf{1} - \mathbf{a}) \oplus \bigoplus_{j=1}^i \mathbf{x}(\tau_j)\right) d\tau 
    &=\sum_{i=1}^{\infty} \frac{1}{i!} \cdot \int_{\tau \in [0,t]^i} f\left((\mathbf{1} - \mathbf{a}) \oplus (1 - (1-\bx)^i)\right) d\tau \\
    &=\sum_{i=1}^{\infty} \frac{1}{i!} \cdot \int_{\tau \in [0,t]^i} f\left(\mathbf{1} - \mathbf{a} \odot (1-\bx)^i\right) d\tau \\
    &=\sum_{i=1}^{\infty} \frac{t^i}{i!} \cdot f\left(\mathbf{1} - \mathbf{a} \odot (1-\bx)^i\right) \ge 0
\end{align*}

\noindent 
By dropping the non-negative sum term from the RHS, we obtain a strictly weaker, but simpler lower bound:
\begin{equation*}
    f\left(\mathbf{1} - \mathbf{a} \odot e^{-z\bx}\right) \ge e^{-t} \cdot f(\mathbf{1} - \mathbf{a})
\end{equation*}

\noindent 
Let $\mathbf{y} = \mathbf{1} - \mathbf{a} \in [0, 1]^d$. 
Let the time parameter $t$ be $z$.
Then,
\begin{equation*}
    f\left(\mathbf{1} - (\mathbf{1} - \mathbf{y}) \odot e^{-z\mathbf{x}}\right) \ge e^{-z} f(\mathbf{y})
\end{equation*}

\noindent
Let $h_z(\mathbf{x}) = \mathbf{1} - e^{-z\mathbf{x}}$. By definition of probabilist sum $\oplus$, $h_z(\mathbf{x}) \oplus \mathbf{y} := \mathbf{1} - (\mathbf{1} - \mathbf{y}) \odot e^{-z\mathbf{x}}$. Thus,
\begin{equation*}
    f\left( h_z(\mathbf{x}) \oplus \mathbf{y} \right) \ge e^{-z} f(\mathbf{y})
\end{equation*}
Define $\bar{x} \triangleq \max_j [\bx]_j$ where $[\bx]_j$ is the $j$-th component of the vector $\bx$. Thus, $\bar{x} \in [0,1]$. Thus, we have
$$
f(h_z(\bx) \oplus \by) \ge e^{-z\bar{x}} f(\by) \geq e^{-z} f(\by)
$$
\end{proof}

\section{Infeasible Projection and Online Gradient Ascent via a Separation Oracle (SO-OGA)}
\label{app:so-oga}

To bypass the computational bottleneck of Euclidean projections in high-dimensional spaces (e.g., $\mathcal{O}(n^3)$ SVD for trace-norm balls), Frank-Wolfe type projection-free methods using Linear Optimization Oracles (LOO) have become standard, following the foundational work of \citet{hazan2012projection} and \citet{jaggi2013revisiting}. However, standard LOO-based methods often face suboptimal convergence rates in adversarial online settings. 
Addressing this, \citet{garber2022new} introduced Separation Oracle (SO) based methods, which can achieve $\mathcal{O}(T^{1/2})$ regret. To efficiently solve our surrogate linear optimization problem, we utilize their \textsc{SO-OGA} algorithm (adapted by \citet{pedramfar2024linear}).
\begin{remark}[Oracles for Constraint Set - LOO and SO]
To circumvent the high cost of Euclidean projections, we rely on two natural projection-free oracles.
Given a convex set $\mathcal{K}$ and a query point $\by$, the Linear Optimization Oracle (LOO) returns $\arg\min_{\mathbf{x} \in \mathcal{K}} \langle \mathbf{y}, \mathbf{x} \rangle$, 
while the Separation Oracle (SO) either asserts $\mathbf{y} \in \mathcal{K}$ or returns a separating hyperplane $\mathbf{g}$ such that $\forall \mathbf{x} \in \mathcal{K}$, $\langle \mathbf{g}, \mathbf{y} - \mathbf{x} \rangle > 0$. 
While LOO is more prevalent, these oracles are \emph{complementary}. As noted by \citet{garber2022new}, for the nuclear norm ball, the LOO is efficient while the SO is expensive. Conversely, for the \textit{spectral norm ball} $\mathcal{B}_2$, the situation is reversed. Thus, the SO enables efficient online learning over domains where the LOO is computationally intractable.
\end{remark}
We utilize the \textsc{SO-OGA} instantiation and its subroutine \textsc{SO-IP} from \citet{pedramfar2024linear}. We adopt the notation from the original paper: let $\mathbf{c} \in \text{int}(\mathcal{K})$ be a center point, $r > 0$ be the radius of a ball contained in $\mathcal{K}$ centered at $\mathbf{c}$, and define the shrunk set $\hat{K}_\delta \triangleq \{ (1-\delta/r)(\mathbf{x} - \mathbf{c}) + \mathbf{c} \mid \mathbf{x} \in \mathcal{K} \}$.

\begin{algorithm}[H]
\caption{Online Gradient Ascent via Separation Oracle - SO-OGA (Algorithm 8 in \citet{pedramfar2024linear})}
\label{alg:so_oga_base}
\begin{algorithmic}[1]
\STATE \textbf{Input:} horizon $T$, constraint set $\mathcal{K}$, step size $\eta$.
\STATE \textbf{Initialize:} $\mathbf{x}_1 \leftarrow \mathbf{c} \in \hat{K}_\delta$.
\FOR{$t = 1, 2, \dots, T$}
    \STATE Play $\mathbf{x}_t$ and observe $\mathbf{o}_t = \nabla f_t(\mathbf{x}_t)$.
    \STATE $\mathbf{x}'_{t+1} = \mathbf{x}_t + \eta \mathbf{o}_t$ \COMMENT{Gradient Ascent Step}
    \STATE Set $\mathbf{x}_{t+1} = \textsc{SO-IP}_{\mathcal{K}}(\mathbf{x}'_{t+1})$. \COMMENT{Output of Algorithm \ref{alg:so_ip}}
\ENDFOR
\end{algorithmic}
\end{algorithm}

\begin{algorithm}[H]
\caption{Infeasible Projection via Separation Oracle - SO-IP$_{\mathcal{K}}(\mathbf{y}_0)$ (Algorithm 9 in \citet{pedramfar2024linear})}
\label{alg:so_ip}
\begin{algorithmic}[1]
\STATE \textbf{Input:} Constraint set $\mathcal{K}$, shrinking parameter $\delta < r$, initial point $\mathbf{y}_0$.
\STATE $\mathbf{y}_1 \leftarrow \mathbf{P}_{\text{aff}(\mathcal{K})}(\mathbf{y}_0)$
\STATE $\mathbf{y}_2 \leftarrow \mathbf{c} + \frac{\mathbf{y}_1 - \mathbf{c}}{\max\{1, \|\mathbf{y}_1\|/D\}}$ \COMMENT{Projection of $\mathbf{y}_0$ over $\mathbb{B}_D^d(\mathbf{c}) \cap \text{aff}(\mathcal{K})$}
\FOR{$i = 1, 2, \dots$}
    \STATE Call Separation Oracle $\text{SO}_{\mathcal{K}}$ with input $\mathbf{y}_i$.
    \IF{$\mathbf{y}_i \notin \mathcal{K}$}
        \STATE Set $\mathbf{g}_i$ to be the hyperplane returned by $\text{SO}_{\mathcal{K}}$ (i.e., $\forall \mathbf{x} \in \mathcal{K}, \langle \mathbf{y}_i - \mathbf{x}, \mathbf{g}_i \rangle > 0$).
        \STATE $\mathbf{g}'_i \leftarrow \mathbf{P}_{\text{aff}(\mathcal{K}) - \mathbf{c}}(\mathbf{g}_i)$
        \STATE Update $\mathbf{y}_{i+1} \leftarrow \mathbf{y}_i - \delta \frac{\mathbf{g}'_i}{\|\mathbf{g}'_i\|}$.
    \ELSE
        \STATE \textbf{Return} $\mathbf{y} \leftarrow \mathbf{y}_i$.
    \ENDIF
\ENDFOR
\end{algorithmic}
\end{algorithm}

\begin{theorem}[Adaptive Regret of SO-OGA, Theorem~9 in \citet{pedramfar2024linear}]
\label{thm:so_oga_regret}
Let $\mathbf{L}$ be a class of linear functions over $\mathcal{K}$ such that $\|l\| \le M_1$ for all $l \in \mathbf{L}$ and let $D = \text{diam}(\mathcal{K})$. Fix $v > 0$ such that $\delta = v T^{-1/2} \in (0, 1)$ and set the step size $\eta = \frac{vr}{2M_1} T^{-1/2}$. Then we have:
\begin{equation}
    \mathcal{AR}_{1, \text{Adv}_1^f(\mathbf{L})}^{\text{SO-OGA}} = O(M_1 T^{1/2}).
\end{equation}
\end{theorem}

\section{Meta-Algorithms and Regret Bounds Towards Diverse Feedback Types}

\subsection{The Generic Meta-Algorithm (OMBQ)}
\label{app:ombq}

We rely on the generic reduction framework established in \citet{pedramfar2024linear}. The following meta-algorithm, OMBQ, transforms a linear learner $\mathcal{A}$ into a non-monotone DR-submodular maximizer using a mapping $h$ and a query oracle $\mathcal{G}$.

\begin{algorithm}[h]
\caption{Online Maximization By Quadratization - OMBQ($\mathcal{A}, \mathcal{G}, h$)}
\label{alg:ombq}
\begin{algorithmic}[1]
\STATE \textbf{Input:} Base algorithm $\mathcal{A}$, Query algorithm $\mathcal{G}$, Mapping $h: \mathcal{K} \to \mathcal{K}$.
\FOR{$t = 1, 2, \dots, T$}
    \STATE Let $\mathbf{x}_t$ be the action chosen by $\mathcal{A}$.
    \STATE \textbf{Play:} $\mathbf{y}_t = h(\mathbf{x}_t)$.
    \STATE \textbf{Query:} Call oracle $\mathcal{G}$ at $\mathbf{x}_t$ to obtain gradient estimate $\mathbf{g}_t$.
    \STATE \textbf{Update:} Pass loss vector $-\mathbf{g}_t$ to $\mathcal{A}$ to update its state.
\ENDFOR
\end{algorithmic}
\end{algorithm}

\begin{theorem}[Regret Transfer, Theorem 1 in \citet{pedramfar2024linear}]
\label{thm:regret_transfer}
Let $\mathcal{A}$ be an algorithm for online optimization with semi-bandit feedback. Also let $\mathcal{F}$ be a function class over $\mathcal{K}$ that is linearizable and surrogate potential $\mathfrak{g}: \mathcal{F} \times \mathcal{K} \to \mathbb{R}^d$ and $h: \mathcal{K} \to \mathcal{K}$. Let $\mathcal{G}$ be a query algorithm for $\mathfrak{g}$ and let $\mathcal{A}' = \text{OMBQ}(\mathcal{A}, \mathcal{G}, h)$.

If $\mathcal{G}$ returns an unbiased estimate of $\mathfrak{g}$ and the output of $\mathcal{G}$ is bounded by $B_1$, then we have:
\begin{equation}
    \mathcal{R}^{\mathcal{A}'}_{\alpha, \text{Adv}_1^o(\mathbf{F}, B_1)} \le \beta \mathcal{R}^{\mathcal{A}}_{1, \text{Adv}_1^f(\mathbf{Q}_\mu[B_1])}
\end{equation}
where $\beta$ is a scaling constant.
\end{theorem}

\subsection{First Order To Zeroth Order (FOTZO)}

\begin{algorithm}[h]
\caption{First order to zeroth order - FOTZO($\mathcal{A}$)}
\label{alg:fotzo_meta}
\begin{algorithmic}[1]
\STATE \textbf{Input:} Shrunk domain $\hat{\mathcal{K}}_\delta$, Linear space $\mathcal{L}_0$, smoothing parameter $\delta \le r$, horizon $T$, algorithm $\mathcal{A}$
\STATE Pass $\hat{\mathcal{K}}_\delta$ as the domain to $\mathcal{A}$
\STATE $k \leftarrow \dim(\mathcal{L}_0)$
\FOR{$t = 1, 2, \dots, T$}
    \STATE $\mathbf{x}_t \leftarrow$ the action chosen by $\mathcal{A}$
    \STATE Play $\mathbf{x}_t$
    \STATE Let $f_t$ be the function chosen by the adversary
    \FOR{$i$ starting from 1, while $\mathcal{A}^{\text{query}}$ is not terminated for this time-step}
        \STATE Sample $\mathbf{v}_{t,i} \in \mathbb{S}^1 \cap \mathcal{L}_0$ uniformly
        \STATE Let $\mathbf{y}_{t,i}$ be the query chosen by $\mathcal{A}^{\text{query}}$
        \STATE Query the oracle at the point $\mathbf{y}_{t,i} + \delta \mathbf{v}_{t,i}$ to get $o_{t,i}$
        \STATE Pass $\frac{k}{\delta} o_{t,i} \mathbf{v}_{t}$ as the oracle output to $\mathcal{A}$
    \ENDFOR
\ENDFOR
\end{algorithmic}
\end{algorithm}

\begin{lemma}[Corollary 4 + Theorem 5 in \citet{pedramfar2024linear}]\label{lem:fotzo_reduction}
Let $\mathbf{F}$ be an $M_1$-Lipschitz function class over a convex set $\mathcal{K}$ and choose $\mathbf{c}$ and $r$ as described above and let $\delta < r$. Let $\mathcal{U} \subseteq \mathcal{K}^T$ be a compact set and let $\hat{\mathcal{U}} = (1 - \frac{\delta}{r})\mathcal{U} + \frac{\delta}{r}\mathbf{c}$. Assume $\mathcal{A}$ is an algorithm for online optimization with first order feedback. Then, if $\mathcal{A}' = \text{FOTZO}(\mathcal{A})$ where FOTZO is described by Algorithm~\ref{alg:fotzo_meta} and $0 < \alpha \le 1$,  we have
\begin{equation*}
    \mathcal{R}^{\mathcal{A}'}_{\alpha, \text{Adv}_0^o(\mathbf{F}, B_0)}(\mathcal{U}) \le \mathcal{R}^{\mathcal{A}}_{\alpha, \text{Adv}_1^o(\hat{\mathbf{F}}, \frac{k}{\delta}B_0)}(\hat{\mathcal{U}}) + \left(3 + \frac{2D}{r}\right)\delta M_1 T.
\end{equation*}
Further, if we have $\mathcal{R}^{\mathcal{A}}_{\alpha, \mathrm{Adv}_1^o(\mathbf{F}, B_1)} = O(B_1 T^\eta)$ and $\delta = T^{(\eta-1)/2}$, then we have
\[
    \mathcal{R}^{\mathcal{A}'}_{\alpha, \mathrm{Adv}_0^o(\mathbf{F}, B_0)} = O(B_0 T^{(1+\eta)/2}).
\]
\end{lemma}

\subsection{Semi-bandit To Bandit (STB)}

\begin{algorithm}[h]
\caption{Semi-bandit to bandit - STB($\mathcal{A}$)}
\label{alg:stb_meta}
\begin{algorithmic}[1]
\STATE \textbf{Input:} Shrunk domain $\hat{\mathcal{K}}_\delta$, Linear space $\mathcal{L}_0$, smoothing parameter $\delta \le r$, horizon $T$, algorithm $\mathcal{A}$
\STATE Pass $\hat{\mathcal{K}}_\delta$ as the domain to $\mathcal{A}$
\STATE $k \leftarrow \dim(\mathcal{L}_0)$
\FOR{$t = 1, 2, \dots, T$}
    \STATE Sample $\mathbf{v}_{t} \in \mathbb{S}^1 \cap \mathcal{L}_0$ uniformly
    \STATE $\mathbf{x}_t \leftarrow$ the action chosen by $\mathcal{A}$
    \STATE Play $\mathbf{x}_t + \delta \mathbf{v}_{t}$
    \STATE Let $f_t$ be the function chosen by the adversary
    \STATE Let $o_t$ be the output of the value oracle
    \STATE Pass $\frac{k}{\delta} o_{t} \mathbf{v}_{t}$ as the oracle output to $\mathcal{A}$
\ENDFOR
\end{algorithmic}
\end{algorithm}

\begin{lemma}[Corollary 5 + Theorem 6 in \citet{pedramfar2024linear}]\label{lem:stb_reduction}
Under the assumptions of Lemma~\ref{lem:fotzo_reduction}, if we assume that $\mathcal{A}$ is semi-bandit, then the same regret bounds hold with $\mathcal{A}' = \text{STB}(\mathcal{A})$, where STB is described by Algorithm 6.
Further, if we have $\delta = T^{-1}$, then $\mathcal{R}^{\mathcal{A}'}_{\alpha, \mathrm{Adv}_0^o(\mathbf{F})}$ has the same order of regret as that of $\mathcal{R}^{\mathcal{A}}_{\alpha, \mathrm{Adv}_1^o(\mathbf{F}, B_1)}$ with $B_1$ replaced with $k M_1$.
\end{lemma}

\subsection{Stochastic Full-information To Trivial query (SFTT)}

\begin{algorithm}[h]
\caption{Stochastic Full-information To Trivial query - SFTT($\mathcal{A}$)}
\label{alg:sftt_meta}
\begin{algorithmic}[1]
\STATE \textbf{Input:} base algorithm $\mathcal{A}$, horizon $T$, block size $L > K$.
\FOR{$q = 1, 2, \dots, T/L$}
    \STATE Let $\hat{\mathbf{x}}_q$ be the action chosen by $\mathcal{A}^{\text{action}}$
    \STATE Let $(\hat{\mathbf{y}}_q^i)_{i=1}^K$ be the queries selected by $\mathcal{A}^{\text{query}}$
    \STATE Let $(t_{q,1}, \dots, t_{q,L})$ be a random permutation of $\{(q-1)L+1, \dots, qL\}$
    \FOR{$t = (q-1)L+1, \dots, qL$}
        \IF{$t = t_{q,i}$ for some $1 \le i \le K$}
            \STATE Play the action $\mathbf{x}_t = \hat{\mathbf{y}}_q^i$
            \STATE Return the observation to the query oracle as the response to the $i$-th query
        \ELSE
            \STATE Play the action $\mathbf{x}_t = \hat{\mathbf{x}}_q$
        \ENDIF
    \ENDFOR
\ENDFOR
\end{algorithmic}
\end{algorithm}

\begin{lemma}[Corollary 6 + Theorem 7 in \citet{pedramfar2024linear}]
\label{lem:sftt_reduction}
Let $\mathcal{A}$ be an online optimization algorithm with full-information feedback and with $K$ queries at each time-step where $\mathcal{A}^{\text{query}}$ does not depend on the observations in the current round and $\mathcal{A}' = \mathrm{SFTT}(\mathcal{A})$. Then, for any $M_1$-Lipschitz function class $\mathbf{F}$ that is closed under convex combination and any $B_1 \ge M_1$, $0 < \alpha \le 1$ and $1 \le a \le b \le T$, let $a' = \lfloor(a-1)/L\rfloor + 1$, $b' = \lceil b/L \rceil$, $D = \mathrm{diam}(\mathcal{K})$ and let $\{T\}$ and $\{T/L\}$ denote the horizon of the adversary. Then, we have
\[
    \mathcal{R}^{\mathcal{A}'}_{\alpha, \mathrm{Adv}_1^o(\mathbf{F}, B_1)\{T\}}(\mathcal{K}_{\star}^T)[a, b] \le M_1 D K (b' - a' + 1) + L \mathcal{R}^{\mathcal{A}}_{\alpha, \mathrm{Adv}_1^o(\mathbf{F}, B_1)\{T/L\}}(\mathcal{K}_{\star}^{T/L})[a', b'],
\]
Further, if we have $\mathcal{R}^{\mathcal{A}}_{\alpha, \mathrm{Adv}_i^o(\mathbf{F}, B)}(\mathcal{K}_{\star}^T)[a, b] = O(B T^\eta)$, $K = O(T^\theta)$ and $L = T^{\frac{1+\theta-\eta}{2-\eta}}$, then we have
\[
    \mathcal{R}^{\mathcal{A}'}_{\alpha, \mathrm{Adv}_i^o(\mathbf{F}, B)}(\mathcal{K}_{\star}^T)[a, b] = O \left( B T^{\frac{(1+\theta)(1-\eta)+\eta}{2-\eta}} \right).
\]
As a special case, when $K = O(1)$, then we have
\[
    \mathcal{R}^{\mathcal{A}'}_{\alpha, \mathrm{Adv}_i^o(\mathbf{F}, B)}(\mathcal{K}_{\star}^T)[a, b] = O \left( B T^{\frac{1}{2-\eta}} \right).
\]
\end{lemma}

\section{Algorithms for Dynamic Regret}
\label{app:dynamic_algos}

To support the dynamic regret guarantees presented in Proposition~\ref{prop:dynamic-regret}, we restate the \textsc{Improved Ader} meta-algorithm and its corresponding expert algorithm from \citet{pedramfar2024linear}. We also include the main theorem governing its performance.

\begin{algorithm}[H]
\caption{Improved Ader - IA (Restated Algorithm 10 from \citet{pedramfar2024linear})}
\label{alg:improved_ader}
\begin{algorithmic}[1]
\STATE \textbf{Input:} horizon $T$, constraint set $\mathcal{K}$, step size $\lambda$, a set $\mathcal{H}$ containing step sizes for experts.
\STATE Activate a set of experts $\{E^\eta \mid \eta \in \mathcal{H}\}$ by invoking Algorithm \ref{alg:expert_instance} for each step size $\eta \in \mathcal{H}$.
\STATE Sort step sizes in ascending order $\eta_1 \le \dots \le \eta_N$, and set $w_{1}^{\eta_i} = \frac{C}{i(i+1)}$ where $C = 1 + \frac{1}{|\mathcal{H}|}$.
\FOR{$t = 1, 2, \dots, T$}
    \STATE Receive $\mathbf{x}_t^\eta$ from each expert $E^\eta$.
    \STATE Play the action $\mathbf{x}_t = \sum_{\eta \in \mathcal{H}} w_t^\eta \mathbf{x}_t^\eta$ and observe $\mathbf{o}_t = \nabla f_t(\mathbf{x}_t)$.
    \STATE Define $\ell_t(\mathbf{y}) := \langle \mathbf{o}_t, \mathbf{y} - \mathbf{x}_t \rangle$.
    \STATE Update the weight of each expert by $w_{t+1}^\eta = \frac{w_t^\eta e^{-\lambda \ell_t(\mathbf{x}_t^\eta)}}{\sum_{\mu \in \mathcal{H}} w_t^\mu e^{-\lambda \ell_t(\mathbf{x}_t^\mu)}}$.
    \STATE Send the gradient $\mathbf{o}_t$ to each expert $E^\eta$.
\ENDFOR
\end{algorithmic}
\end{algorithm}

\begin{algorithm}[H]
\caption{Improved Ader : Expert algorithm (Restated Algorithm 11 from \cite{pedramfar2024linear})}
\label{alg:expert_instance}
\begin{algorithmic}[1]
\STATE \textbf{Input:} horizon $T$, constraint set $\mathcal{K}$, step size $\eta$.
\STATE Let $\mathbf{x}_1^\eta$ be any point in $\mathcal{K}$.
\FOR{$t = 1, 2, \dots, T$}
    \STATE Send $\mathbf{x}_t^\eta$ to the main algorithm.
    \STATE Receive $\mathbf{o}_t$ from the main algorithm.
    \STATE \textbf{Update:} $\mathbf{x}_{t+1}^\eta = \mathbf{P}_{\mathcal{K}}(\mathbf{x}_t^\eta + \eta \mathbf{o}_t)$
    \STATE \COMMENT{\textit{Note: To maintain the projection-free property of our framework, we implement this update using the Frank-Wolfe step or the Infeasible Projection subroutine from Algorithm \ref{alg:so_ip}.}}
\ENDFOR
\end{algorithmic}
\end{algorithm}

\begin{theorem}[Dynamic Regret of Improved Ader, Theorem 10 in \citet{pedramfar2024linear}]
\label{thm:dynamic_regret_bound}
Let $\mathbf{L}$ be a class of linear functions over $\mathcal{K}$ such that $\|\ell\| \le M_1$ for all $\ell \in \mathbf{L}$ and let $D = \text{diam}(\mathcal{K})$. Set $\mathcal{H} := \{\eta_i = \frac{2^{i-1}D}{M_1}\sqrt{\frac{7}{2T}} \mid 1 \le i \le N\}$ where $N = \lceil \frac{1}{2} \log_2(1 + 4T/7) \rceil + 1$ and $\lambda = \sqrt{2/(TM_1^2 D^2)}$. Then for any comparator sequence $\mathbf{u} \in \mathcal{K}^T$, we have
\begin{equation}
    \mathcal{R}^{\text{IA}}_{1, \text{Adv}_1^f(\mathbf{L})}(\mathbf{u}) = O\left(M_1 \sqrt{T(1 + P_T(\mathbf{u}))}\right)
\end{equation}
where $P_T(\mathbf{u}) = \sum_{t=1}^T \|\mathbf{u}_t - \mathbf{u}_{t-1}\|_2$ is the path length of the comparator sequence.
\end{theorem}

\end{document}